\documentclass{article}

\usepackage[top=1in, bottom=1in, left=1in, right=1in]{geometry}

\usepackage{microtype}
\usepackage{graphicx}
\usepackage[numbers]{natbib}
\usepackage{booktabs} 

\usepackage{hyperref}
\usepackage{color}

\usepackage{microtype}
\usepackage{subfigure}
\usepackage{booktabs} 

\usepackage{graphicx}
\usepackage{parskip}
\usepackage{breakcites}
\usepackage{amssymb}
\usepackage{amsmath}
\usepackage{multirow}
\usepackage{arydshln}
\usepackage{physics}
\usepackage{amsthm}
\usepackage{bbm}
\usepackage{xcolor}

\newcommand\Z{\mathbb Z}

\newcommand\R{\mathbb R}
\newcommand\E{\mathbb E}

\newtheorem{theorem}{Theorem}[section]

\newtheorem{remark}[theorem]{Remark}
\newtheorem{proposition}[theorem]{Proposition}

\def\abs#1{\left| #1 \right|}
\renewcommand{\norm}[1]{\ensuremath{\left\lVert #1 \right\rVert}}
\renewcommand{\norm}[1]{\lVert #1 \rVert}

\def\h_#1{\hat{#1}}
\def\wh_#1{\widehat{#1}}

\newcommand{\specialcell}[2][c]{%
  \begin{tabular}[#1]{@{}l@{}}#2\end{tabular}}

\newcommand\set[1]{\left\{#1\right\}} 

\newcommand{\vk}[1]{\textcolor{black}{#1}}

\title{\bf Towards Unbiased and Accurate Deferral to Multiple Experts}
\author{Vijay Keswani\footnote{Work done while at Amazon.} \\ Yale University \and Matthew Lease \\ University of Texas at Austin\\ Amazon AWS AI \and Krishnaram Kenthapadi \\ Amazon AWS AI}
\date{}

\begin{document}
\maketitle




\begin{abstract}
    Machine learning models are often implemented in cohort with humans in the pipeline, with the model having an option to defer to a domain expert in cases where it has low confidence in its inference. 
    Our goal is to design mechanisms for ensuring accuracy and fairness in such prediction systems that combine machine learning model inferences and domain expert predictions. 
    Prior work on ``deferral systems'' in classification settings has focused on the setting of a pipeline with a single expert and aimed to accommodate the inaccuracies and biases of this expert to simultaneously learn an inference model and a deferral system. 
    Our work extends this framework to settings where multiple experts are available, with each expert having their own domain of expertise and biases. We propose a framework that simultaneously learns a classifier and a deferral system, with the deferral system choosing to defer to one or more human experts in cases of input where the classifier has low confidence. 
    We test our framework on a synthetic dataset and a content moderation dataset with biased synthetic experts, and show that it significantly improves the accuracy and fairness of the final predictions, compared to the baselines.
    We also collect crowdsourced labels for the content moderation task to construct a real-world dataset for the evaluation of hybrid machine-human frameworks and show that our proposed framework outperforms  baselines on this real-world dataset as well.
\end{abstract}

\maketitle
\clearpage
\tableofcontents
\clearpage

\section{Introduction} \label{sec:introduction}
Real-world applications of machine learning (ML) models often involve the model working together with human experts \cite{halfaker2019ores,de2020case}.
For example, a model that predicts the likelihood of a disease given patient information can choose to defer the decision to a doctor who can make a relatively more accurate diagnosis \cite{kieseberg2016trust, raghu2019direct}.
Similarly, risk 
assessment tools work together with judges and domain experts to provide a baseline recidivism risk estimate \cite{green2019disparate, duwe2017effects}. 
%
Other examples of such hybrid decision-making settings include  financial analysis tools \cite{zetzsche2020artificial} and content moderation tools for abusive speech detection \cite{olteanu2017limits} and fake news identification \cite{fake_news}.

Human-in-the-loop frameworks are often employed in settings where automated models cannot be trusted to have high quality inferences for all kinds of inputs.
Beyond the incentive of improved overall accuracy, having human experts in the pipeline also ensures timely audits of the predictions \cite{sutton2018digitized} and helps fill gaps in the training of the automated models \cite{patterson2013bootstrapping, liu2019deep}.
%
%
A case in point is the study done by \citet{chouldechova2018case} which showed that erroneous risk assessments by a child maltreatment hotline screening tool were frequently flagged as being incorrect by the human reviewers, implying that automated tools may not always cover the entire feature space that the domain experts use to make the decision.

%
However, the interaction between an ML model and a human expert is inherently more complicated than an entirely-automated pipeline. 
Prior studies on settings where human-in-the-loop frameworks have been implemented provide evidence of such complexities \cite{cummings2004automation,alberdi2009people,goddard2012automation,narayanan2018humans}.
One serious complication is the possibility of aggravated biases against protected groups, defined by attributes such as gender and race.
With increasing utilization of ML in \textit{human classification} tasks, the problem of biases against protected groups in automated predictions has received a lot of interest.
This has led to a deep exploration of social biases in popular models/datasets and ways to algorithmically mitigate them \cite{barocas_fairness_2019, mehrabi2019survey}.
Nevertheless, a number of such biased models and datasets are still in use \cite{noble2018algorithms}.
In a pipeline that involves an interaction between a possibly-biased ML model and a human, the biases of the human can aggravate the biases of the model \cite{forough2020}.
For example, in a study by \citet{green2019disparate}, participants were given the demographic attributes and prior criminal record of various defendants, along with the model-predicted risk of recidivism associated with each defendant, and asked to predict the risk.
They found that the participants associated a higher risk with black defendants, compared to the model prediction.
In this case, the possible biases of the human in the pipeline seem to exacerbate the bias of the model prediction.
Similar ethical concerns regarding the interplay between the biases of model and humans have been highlighted in other papers \cite{chouldechova2017fair, raji2020saving}.

Motivated by the challenges discussed above, we focus on mechanisms for ensuring accuracy and fairness in hybrid machine-human pipelines.
We consider the setting where a classification model is trained to either make a decision or defer the decision to human experts. 
%
Most machine-human pipelines employed in real-world applications have multiple human experts available to share the load and to cover different kinds of input samples \cite{chouldechova2018case, gronsund2020augmenting}.
Therefore, the hybrid decision-making framework will have an additional task of appropriately choosing one or more experts when deferring.
%
Each expert may also have their own area of expertise as well as possible biases against certain protected groups, characterized by their prior predictions on some samples.
Correspondingly, the training of a machine learning model in such a composite pipeline has to take into account the domain expertise of the humans, and delegate the prediction task in an input-specific manner.
%
Hence, our goal is to train a classifier and a deferral system such that the final predictions of the composite system are accurate and unbiased.


\textbf{Our Contributions. }
%
%
We study the multiple-experts deferral setting for classification problems, and present a formal \textit{joint learning framework} that aims to simultaneously learn a classifier and a \textit{deferrer}. The job of the deferrer is to select one or more experts (including the classifier) to make the final decision (\S\ref{sec:main_framework}).
As part of the framework, we propose loss functions that capture the costs associated with any given classifier and deferrer.
We theoretically show that, given prior predictions from the human experts and true class labels for the training samples, the proposed loss functions can be optimized using gradient-descent algorithm to obtain an effective classifier and deferrer.
%
Our framework further supports the settings where (a) number of experts that can be consulted for each input is limited, (b) each expert has an individual cost of consultation, and/or (c) expert predictions are available for only a subset of training samples (\S\ref{sec:variants}).
To ensure that the final predictions are unbiased with respect to a given protected attribute, we propose two fair variants of the framework (\textit{joint balanced} and \textit{joint minimax-fair}) that aim to improve error rates across all protected groups.
Our framework can handle both multi-class labels and non-binary protected attributes. 

We empirically demonstrate the efficacy of our framework and its variants on multiple datasets: a synthetic dataset constructed to highlight the importance of simultaneously learning a classifier and a deferrer (\S\ref{sec:synthetic_dataset}), an offensive language dataset \cite{davidson2017automated} with synthetically-generated experts (\S\ref{sec:offensive_synthetic}), and a real-world dataset constructed to specifically evaluate deferral frameworks with multiple available experts (\S\ref{sec:mturk}).
The real-world dataset consists of a large number of crowdsourced labels for the offensive language dataset,
and is also a contribution of this paper.
Unlike most crowdsourced datasets where the goal is simply to obtain accurate annotations, this dataset explicitly contains a dictionary of crowdworker (anonymized) to predicted labels, ensuring that the decision-making ability of each crowdworker can be inferred and consequently used to evaluate the performance of a hybrid framework like ours.
We plan to publish this dataset
as this will provide a strong empirical benchmark to foster future work.
%
For all datasets, our framework significantly improves the accuracy of the final predictions (compared to just using a classifier and other baselines, such as task allocation algorithms of \citet{li2015cheaper} and \citet{qiu2016crowdselect} from crowdsourcing literature), and for the offensive language datasets, the fair variants of the framework also reduce disparity across the dialect groups.

\textbf{Related Work. }
Given the difficulty of constructing and analyzing a human-in-the-loop framework, prior work has looked at human-in-the-loop settings from various viewpoints. 
%
%
One direction of research has explored the idea of the classifier having a ``reject''/``pass'' option for contentious input samples \cite{el2010foundations, li2011knows, cortes2016learning, jung2014predicting, liu2019deepgamblers, cortes2016boosting,cortes2018online}.
While such an option is usually provided to ensure that low confidence decisions can be deferred to human experts, the penalty of abstaining from making a decision in these models is fixed, and therefore, they do not take into account whether the expert at the end of the pipeline has the relevant knowledge to make the decision or not.

On the other hand, papers that take the biases and/or accuracies of the human experts into consideration are inherently more robust, but also more difficult to train and analyze.
Prior theoretical models for learning to defer \cite{madras2018predict, mozannar2020consistent, raghu2019algorithmic, de2020regression, wilder2020learning} have constructed explicit loss functions/optimization methods to model the combined inaccuracies and biases of the classifier and the human expert.
Unlike the classifiers with reject option, they use a non-static loss function for the human expert and ensure that the penalty of deferring to a human expert is input-specific. 
%
%
However, \cite{madras2018predict, mozannar2020consistent, de2020regression, wilder2020learning, bansal2020optimizing} work with a single expert, assuming that the expert in the pipeline will be fixed and remain the same for future classification.
Such an assumption is inhibitory in the settings where multiple experts are available \cite{chouldechova2018case}, as different human experts can have different prediction behaviours \cite{guan2018said}.
%
%
\citet{raghu2019algorithmic} model an optimization problem for the hybrid setting as well, but they learn a classifier and a deferrer separately, which (as shown by \cite{mozannar2020consistent} and discussed in \S\ref{sec:synthetic_experiments}) cannot handle a large variety of input settings since the classifier does not adapt to the experts.
In comparison, our method learns a classifier and a deferrer simultaneously, and can handle multiple experts.
%

Empirical studies in this direction \cite{green2019disparate, zhang2020effect, de2020case, katell2020toward, chouldechova2018case, katell2020toward} often inherently use multiple experts since the results are based on crowdsourced data, but do not aim to propose a learning model for the pipeline.
They, however, do highlight the importance of taking the domain knowledge of experts into account to improve the accuracy and fairness of the entire pipeline.

Another field that studies the problem of \textit{task allocation} among different humans is \textit{crowdsourcing}.
Crowdsourcing for data collection is a popular approach to label or curate different kinds of datasets \cite{lease2011quality}.
%
Since crowdworkers employed for such annotation tasks come from diverse backgrounds, prior work in crowdsourcing has looked at the related issue of efficient distribution of input amongst the available workers \cite{nguyen2015combining, yan2011active,patterson2013bootstrapping,venanzi2014community,kamar2015identifying,nushi2015crowd,li2015cheaper,qiu2016crowdselect, tu2020multi}.
\vk{The main difference between this line of work and our setting is the presence of the automated classifier.
%
In our setting, 
the classifier is expected to handle the primary load of prediction tasks
and the role of human experts is to provide assistance for input samples where the classifier cannot achieve reasonable confidence. 
Crowdsourcing models, however, do not usually involve construction of any prediction model.
}
\vk{One can alternately pre-train the classifier and treat it as another crowdworker to use task-allocation algorithms from crowdsourcing literature to distribute the samples among the experts.
The main issue with this approach is that training the classifier and deferrer separately can lead to an ineffective prediction pipeline. 
%
%
In our empirical analysis (\S\ref{sec:synthetic_experiments}), we assess the performance of two task-allocation algorithms from crowdsourcing literature \cite{li2015cheaper, qiu2016crowdselect}, and demonstrate the necessity of simultaneous training.
See Appendix~\ref{sec:baselines} for detailed discussion on these crowdsourcing methods.%
%
}

\section{Model}

%
Each sample in the domain contains a class label, denoted by $Y \in \mathcal{Y}$, 
$n$-dimensional feature vector (default attributes) of the sample used to predict the class label, 
denoted by $X \in \mathcal{X}$, and additional information about the sample that is available only to the experts, denoted by $W \in \mathcal{W}$.
$W$ can represent different human factors that often assist in decision-making, such as training or background of the expert for the given task.
Let $\Delta_Y$ denote the vertices of the simplex corresponding to the unique class labels in $\mathcal{Y}$ and let $\text{conv}(\Delta_Y)$ denote the simplex and its interior.
Every sample also has a protected attribute $Z \in \mathcal{Z}$ associated with it (e.g., gender or race); $Z$ can be part of default attributes $X$ or additional attributes $W$, depending on the context.
%
%

Our framework consists of a classifier and a deferrer.
The classifier $F : \mathcal{X}{\rightarrow}\text{conv}(\Delta_Y)$, given the default attributes of an input sample, returns a probability distribution over the labels of $\mathcal{Y}$.
Let $L_{\text{clf}}(F; X,Y)$ denote the convex loss associated with the prediction of classifier $F$ at point $(X,Y)$.
For $\ell > 0$, we will call $L_{\text{clf}}$ an $\ell$-Lipschitz smooth function if for all classifiers $F$, $\nabla_F^2 \left(\E_{X,Y} L_{\text{clf}}(F; X,Y) \right) \preccurlyeq \ell \cdot \mathbf{I}$.
Intuitively, Lipschitz-smoothness characterizes how fast the gradient of $L_{\text{clf}}$ changes around any point in the parameter space of the classifier; this characterization crucially helps determine the step-size required for the gradient-descent optimization of the loss function and will be useful for convergence rate bounds in our setting as well.
%

The framework also has access to $m-1$ human experts $E_1, \dots, E_{m-1} : \mathcal{X} \times \mathcal{W}\rightarrow \Delta_Y$ who can assist with the decision-making.
The output of the expert will be a vector with 1 for the index of the predicted class and 0 for all other indices (one-hot encoding).
%
The experts are assumed to have access to the additional information (from domain $\mathcal{W}$) that can be used to make the predictions more accurately; however, deferring to an expert will come at an additional cost which we will quantify later.
We also assume that there is an \textit{identity expert} which just returns the decision made by the classifier $F$; therefore, in total we have $m$ experts ($E_m(X, W)= F(X)$) (see Figure~\ref{fig:model} in Appendix).
For any given input $X$, the following notation will denote all the decisions,
\[\textstyle Y_E(X, W):= [E_1(X,W)^\top, \ldots,  E_{m-1}(X,W)^\top, F(X)^\top].\]
The goal of the deferral system $D: \mathcal{X}\rightarrow \set{0,1}^m$, given the input, is to defer to one or more experts (including the classifier) who are likely to make accurate decision for the given input.
Given any input, $D$ will choose a committee of experts and the final output of the framework will be based on the entries of the following matrix-vector product: $Y_E(X, W) \cdot D(X)$ (the specific aggregation method used is specified in the \S\ref{sec:main_framework}).
%
%
If the committee chosen contains only the \textit{identity expert}, then the output of the framework is the output of the classifier $F$; otherwise, the output of the model is the aggregated decision of the chosen committee.



\begin{remark}
The difference between a human-in-the-loop setting and setting with composition of multiple prediction models \cite{dwork2018fairness, bower2017fair, chen2020frugalml} is the access to additional information $W$.
$W$ represents the decision-making assistance available to the experts that is not available to the prediction model either due to computational limits on the prediction model or due to lack of availability of this data for training.
This assumption crucially implies that, in most cases, we cannot construct a \textit{suitably-accurate} model to simulate the predictions of the experts since the importance assigned to the additional information $W$ is unknown.
In the absence of $W$, one can only try to identify the input samples for which the expert is expected to be more accurate than the trained classifier; identifying such input samples using $X$ is exactly the job of the deferrer in our framework.
%
This 
distinction
separates our problem setting from one where expert labels are used to bootstrap a classifier \cite{patterson2013bootstrapping}.
%
\end{remark}

\subsection{Simultaneous Learning Classifier \& Deferrer} \label{sec:main_framework}

%
%
%
We first present our framework for the case of binary class label and later discuss the extension to multi-class setting.

\textbf{Binary class label}, i.e, $\mathcal{Y} = \set{0,1}$.
%
%
Suppose the classifier $F$ is fixed and, 
given the $m$ experts, we need to provide a mechanism for training the deferral system (we will generalize this notion for simultaneous training shortly).
%
%
For any given input $X$, the deferrer output $D(X)$ is expected to be a vector in the discrete domain $\set{0,1}^m$.
For the sake of smooth optimization, we will relax the domain of the output of $D$ to include the interior of the hypercube $[0,1]^m$, i.e., $D(X)$ will quantify the weight associated with each expert, for the given input $X$.
Since we consider the binary class label setting, we can simplify our notation further for this section. Let $Y_{E,1}(X, W)$ denote the second row of the $2 \times m$ matrix $Y_E(X, W)$; this simplification does not lead to any loss of representational power since the sum of first and second row is the vector $\mathbf{1}$.
%
Along similar lines as logistic regression, using $D(X)$ one can then directly calculate the output prediction (probabilistic) as follows:
$\textstyle \hat{Y}_D := \sigma (D(X)^\top Y_{E,1}(X, W)),$
where $\sigma(x) := e^x/(e^x + e^{1-x})$.
We can then train the deferrer to optimize the standard log-loss risk function:
\[\textstyle \min_D - \E_{X, Y}\left[ Y \log(\hat{Y}_D) + (1-Y)\log (1 - \hat{Y}_D)\right].\]
%
%
The expectation is over the underlying distribution; the empirical risk can be computed as mean of losses over any given dataset samples (i.e., expectation over empirical distribution).
For any input sample, the output prediction of the framework is 1 if $\sigma (D(X)^\top Y_{E,1}(X, W))> 0.5$ else 0.

While the above methodology trained $F$ and $D$ separately, we can combine the training of the two components as well.
%
To train $F$ and $D$ simultaneously, we introduce hyper-parameters $\alpha_1, \alpha_2$, and merge the loss functions for the classifier $F$ and deferrer $D$ linearly using these hyperparameters.
\begin{align*}
\textstyle L(F,D)= \alpha_1\E_{X, Y} \left[ L_{\text{clf}}(F; X,Y)\right] - \alpha_2 \E_{X, Y} \left[ Y \log(\hat{Y}_D) + (1-Y)\log (1 - \hat{Y}_D)\right].    
\end{align*}


The choice of hyperparameters is context-dependent and is discussed later.
The goal of the framework is then to find the (classifier, deferrer) pair that optimizes
$\min_{F,D} L(F, D).$
%
We will refer to this model as the \textit{joint framework}.
%
%
%
The joint learning framework extends the standard logistic regression method, and hence, exhibits some desirable properties.

\begin{proposition} \label{prop:convex}
$L(F,D)$ is convex in $F$ and $D$, given a convex $L_{\text{clf}}$.
\end{proposition}
The convexity of the function enables us to use standard gradient-descent optimization approaches \cite{boyd2004convex} to optimize the loss function.
In particular, we will use the projected-gradient descent algorithm, with updates of the following form:
\[ F_{t+1} = F_t - \eta \cdot \eval{\pdv{L}{F}}_{F = F_t}, \]
\[ D_{t+1} = \text{proj}_{\set{0,1}^m} \left(D_t - \eta \cdot \eval{\pdv{L}{D}}_{D = D_t}\right),\]
where $\eta > 0$ is the learning rate and $\text{proj}_{\set{0,1}^m} (\cdot)$ operator projects a point to its closest point in the hypercube $\set{0,1}^m$.
Further, we can show that the gradient of the loss function assigns relatively larger weight to more accurate experts.
\begin{theorem}[Deferrer gradient updates] \label{thm:gradient}
Suppose that $\alpha_1, \alpha_2$ are independent of the parameters of $D$. Let $Y_E \in \set{0,1}^m$ denote the decisions of the experts and classifier for any given input, and let $Y$ denote the class label for this input. Then, for any $i \in \set{1, \ldots, m}$,
\[ - \pdv{L}{D}^{(i)} \propto
\begin{cases}
e^{1 - D^\top Y_{E,1}}, & \text{ if } Y = 1, Y = Y_{E,1}^{(i)}, \\
- e^{D^\top Y_{E,1}}, & \text{ if } Y = 0, Y \neq Y_{E,1}^{(i)}, \\
0, & \text{ otherwise. }
\end{cases}\]
Here $u^{(i)}$ denotes the $i$-th element of vector $u$.
\end{theorem}
%

%
The above theorem states that gradient descent moves in a direction that rewards more accurate experts.
Conditional on $Y=1$, the difference between the weight updates of a correct and an incorrect expert is proportional to $e^{1 - D^\top Y_{E,1}}$. 
Similarly, conditional on $Y=0$, the difference between the weight updates of a correct and an  incorrect expert is proportional to $e^{D^\top Y_{E,1}}$. 
We next provide convergence bounds for the projected gradient descent algorithm in our setting when $L_{\text{clf}}$ is Lipschitz-smooth and $\alpha_1, \alpha_2$ are constants.

\begin{theorem}[Convergence bound] \label{thm:convergence}
Suppose $L_{\text{clf}}$ is $\ell$-Lipschitz smooth and $\alpha_1, \alpha_2$ are constants.
Let $(F^\star, D^\star):= \arg\min_{F,D} L(F, D)$. Given starting point $F_0$, such that $\norm{F_0{-}F^\star}\leq\delta$, step size $\eta=c(\ell{+}m)^{-1}$, for an appropriate constant $c>0$, and $\varepsilon>0$, the projected-gradient descent algorithm, after $T = O\left( \frac{ (\ell + m)(\delta^2 + m)}{\varepsilon}  \right)$ iterations, returns a point $F^\circ, D^\circ$, such that
$L(F^\circ, D^\circ) \leq L(F^\star, D^\star) + \varepsilon.$
\end{theorem}

\vk{
Note that for $m{=}1$ (just the classifier), we recover the standard gradient descent convergence bound for $\ell$-Lipschitz smooth loss function $L_{\text{clf}}$, i.e., $O(\ell\delta^2/\varepsilon)$ iterations \cite{boyd2004convex}.
%
For $m{>}1$, additionally finding the optimum deferrer results in an extra $(m(\delta^2{+}\ell){+}m^2)/\varepsilon$ additive term.
%
%
With standard classifiers and loss functions, we can use the above theorem to get non-trivial convergence rate bounds.
%
For example,
%
if $F$ is a logistic regression model and $L_{\text{clf}}$ is the log-loss function, Lipschitz-smoothness parameter $\ell$ is the maximum eigenvalue of the feature covariance matrix.
%
}
Our theoretical results show that, given prior predictions from the experts and true class labels for a training set, loss function $L$ can be used to train a classifier and an effective deferrer using gradient descent.
Appendix~\ref{sec:proofs} contains the proofs of all theoretical results.

\textbf{Multi-class label. }
The above framework can be extended to multi-class settings as well.
In this case, the matrix-vector product $Y_E(X, W) \cdot D(X)$ is a $|\mathcal{Y}|$-dimensional vector.
Similar to the binary case, we extract the probability of every class label and represent it using $\hat{Y}_D$, where the $j$-th coordinate of $\hat{Y}_D$ represents the probability of class label being $j$, i.e,
\[ \hat{Y}_D^{(j)} := \frac{e^{D(X)^\top Y_{E,j}(X, W)}}{ \sum_{j'=1}^{|\mathcal{Y}|} e^{D(X)^\top Y_{E,j'}(X, W)}}.\]
The loss function $L(F,D)$ in this case can be written as 
\begin{align*}
\textstyle \alpha_1\E_{X, Y} \left[ L_{\text{clf}}(F; X,Y)\right]{-}\alpha_2 \E_{X, Y} \left[ \sum_{j{=}1}^{|\mathcal{Y}|} \mathbbm{1}[Y{=}j] \log \hat{Y}_D^{(j)}\right].
    \end{align*}


The final output of the framework, for any given input, is $\arg \max \hat{Y}_D$.
The above loss function retains the desired properties from the binary setting; it is convex with respect to the classifier and deferrer, and the indicator formulation ensures that each gradient step still rewards the experts that are correct for any given training input.
%
%
Additional costs considered in cost sensitive learning \cite{zhou2010multi}, e.g., different penalties for different incorrect predictions, can also be incorporated in our framework by simply replacing the indicator function $\mathbbm{1}[Y = j]$ with the penalty function \cite{mozannar2020consistent}.
For the sake of simplicity, we omit those details from this version of the paper.

\textbf{Choice of hyperparameters. }
$\alpha_1$ and $\alpha_2$ can either be kept constant or chosen in a context dependent manner.
First, note that, since $\hat{Y}_D$ includes the classifier decision as well (scaled by the weight assigned to the classifier), keeping $\alpha_1 = 0$ would also ensure that the classifier and deferrer are trained simultaneously.
%
%
%
However, due to the associated weight, classifier training with  $\alpha_1 = 0$ can be slow and, since the initial classifier parameters are untrained, the classifier predictions in the initial training steps can be almost random.
This will lead to the deferrer assigning low weight to the classifier.
Correspondingly, depending on the complexity of the prediction task, it may be necessary to give the classifier a head-start as well.
One way is to use time-dependent $\alpha_1, \alpha_2$.
%
set $\alpha_1=1$ and $\alpha_2 = 1 - t^{-c}$, where $t \in \Z_{+}$ is the training iteration number and $c > 0$ is a constant.
This choice ensures that in the initial iterations $F$ is trained primarily and in the later iterations $F$ and $D$ are trained simultaneously.

There is a natural tradeoff associated with this head-start approach as well.
The simultaneous training of $F$ and $D$ is crucial because the goal is to defer to experts for input where the classifier cannot make an accurate decision without the additional information.
Therefore, a large head-start for the classifier can lead to a sub-optimal framework if the classifier tries to improve its accuracy over the entire domain.\footnote{The synthetic experiment in \S\ref{sec:synthetic_dataset} and the examples in \citet{mozannar2020consistent} (for a single expert setting) highlight the necessity of simultaneously learning the classifier and deferrer.}
Another choice of hyperparameters that can address this domain-partition setting is the following: set $\alpha_1{=}1$ and $\alpha_2{=}\mathbbm{1}[ \arg \max F(X){\neq}Y]$ so that the deferrer is trained on training samples for which the classifier is incorrect.
%
%

\subsection{Variants of the Joint Framework} \label{sec:variants}
%
We propose several variants of the joint learning framework
that are inspired by the real-world problems that a human-in-the-loop model can encounter.


%
\textbf{Fair Learning. }
The above joint framework aims to use the ability of the experts to ensure that the final predictions are more accurate than just the classifier.
However, a possible pitfall of this approach can be that it can exacerbate the bias of the classifier, with respect to the protected attribute $Z$. 
Prior work has shown that misrepresentative training data \cite{buolamwini2018gender, kay2015unequal} or inappropriate choice of model \cite{noble2018algorithms}, along with the biases of the human experts \cite{green2019disparate, sap2019risk} can lead to disparate performance across protected attribute types.
An example of such disparity in our setting would be when, in an attempt to decrease the error rate of the prediction, the joint framework assigns larger weights to the biased experts, leading to an increase in disparity of predictions with respect to the protected attribute.
%
%
We provide two approaches to handle the possible biases in our framework and ensure that the final predictions are fair.

\textit{Balanced Error Rate. }
One way to address the bias in final predictions is to give equal importance to all protected groups in our loss function.
For protected attribute type $z$, let
%
\begin{align*}
\textstyle L^z(F, D) := \alpha_1\E_{X, Y \mid Z=z} \left[ L_{\text{clf}}(F; X,Y)\right]- \alpha_2 \E_{X, Y \mid Z=z} \left[ Y \log(\hat{Y}_D) + (1-Y)\log (1 - \hat{Y}_D)\right].    
\end{align*}
Then the goal of this fair framework is to find the optimal solution for
$\min_{F,D} \sum_{z \in \mathcal{Z}} L^z(F, D).$
The above method is also equivalent to assigning group-specific weights to the samples \cite{kamiran2009classifying,friedler2019comparative}.
%
We will refer to this framework as the \textit{joint balanced framework}.
%


\textit{Minimax Pareto Fairness. }
\citet{martinezminimax}'s proposed Pareto fairness aims to reduce disparity by minimizing the worst error rate across all groups.
In other words, minimax Pareto fairness proposes solving the following optimization problem: $\min_{F,D} \max_{z \in \mathcal{Z}} L^z(F, D).$

We will employ this fairness mechanism as well and refer to this framework as the \textit{joint minimax-fair framework}.
To understand the intuition behind this framework, we theoretically show that, in case of a binary protected attribute, the solution to the minimax Pareto fair program reduces the disparity between the risks across the protected attribute types.
%
\begin{theorem}[Disparity of minimax-fair solution] \label{thm:minimax_fair}
Suppose $\mathcal{Z} = \set{0,1}$. Let $F^\star, D^\star := \arg\min_{F,D} \max_{z \in \mathcal{Z}} L^z(F, D)$ denote the joint minimax-fair framework optimal solution and let $F^\circ, D^\circ := \arg\min_{F,D} L(F, D)$ denote the joint framework optimal solution.
Then
\[\textstyle \abs{L^0(F^\star, D^\star){-}L^1(F^\star, D^\star)} \leq \abs{L^0(F^\circ, D^\circ){-}L^1(F^\circ, D^\circ)}.\]
\end{theorem}
%
%
The proof is presented in Appendix~\ref{sec:proofs}.
Note that minimax Pareto fairness is a generalization of fairness by balancing error rate across the protected groups, but is also more difficult and costly to achieve.
Furthermore, minimax Pareto fairness can handle non-binary protected attributes as well; we refer the reader to \citet{martinezminimax} for further discussion on the properties of the minimax-fair solution.
For our simulations, we will use the algorithm proposed by \cite{diana2021Convergent} to achieve minimax Pareto fairness.

Depending on the application, other fairness methods can also be incorporated in the framework.
For example, if the fairness goal is to ensure demographic parity or equalized odds, then fairness constraints \cite{dwork2018fairness,celis2019classification}, regularizers \cite{kamishima2011fairness}, or post-processing methods \cite{hardt2016equality, pleiss2017fairness} can alternately be employed.

\textbf{Sparse Committee Selection. } 
The joint framework could assign non-zero weight to all experts.
%
In a real-world application, requiring predictions from all of the experts can be extremely costly.
%
%
To address this, we propose a sparse variant to choose a limited number of experts per input.

The number of experts consulted for any given input can be limited by using the weights from $D(X)$ to construct a small committee.
Suppose we are given that the committee size can be at most $k$. Then, for any input $X$, we construct a probability distribution over the experts with probability assigned to each expert being proportional to its weight in $D(X)$,
and sample $k$ experts i.i.d. from this distribution.
The final output can be obtained by replacing $D^\top Y_{E}$ in $\hat{Y}_D$ by the mean prediction of the committee formed by this subset (scaled by the sum of weights in $D$).
%
%
We refer to this framework as the \textit{joint sparse framework}, when using the simple log-loss objective function, or \textit{joint balanced/minimax-fair sparse framework}, when using either balanced or minimax-fair log-loss objective function. 
We can show that the expected error disparity between joint normal and joint sparse solutions indeed depends on the properties of the distribution induced by $D(X)$.
\begin{theorem}[Price of sparsity] \label{thm:sparsity}
Suppose $\mathcal{Y} = \set{0,1}$ and let $D$ denote the deferrer output and $\hat{Y}_D$ denote the prediction of the joint framework for a given input. 
Given $k \in [m]$, let random variable $\tilde{Y}_{D,k}$ denote the prediction of the joint \textbf{sparse} framework for this input.
%
The expected difference of loss across the two predictions can be bounded as follows:
\[\textstyle \E \abs{\log\hat{Y}_D{-}\log \tilde{Y}_{D,k}} < s_D \norm{D}_1  + \max\left(2\norm{D}_1, 1\right),\]
where $s_D$ denotes the mean absolute deviation \cite{gorard2005revisiting} of the distribution induced by $D$.
\end{theorem}
%
$s_D$ characterizes the dispersion of the distribution induced by $D$ and if $D$ has low dispersion, then the expected difference of loss from choosing a committee from distribution induced by $D$ is low.
The proof is presented in Appendix~\ref{sec:proofs}.
%
One could also, alternately, select the experts with the $k$-largest weights for each input  \cite{jung2013crowdsourced}.




\textbf{Dropout. }
Given the possible disparities in the accuracies of the experts at the end of the pipeline, training a joint learning framework with diverse experts can suffer from the generalization pitfalls seen commonly in optimization literature \cite{mohri2018foundations}. 
If one expert is relatively more accurate than other experts the framework can learn to assign a relatively larger weight to this expert for every input compared to other experts.
This is, however, quite undesirable as it 
assigns a disproportionate load to just one (or a small subset) of experts.

To tackle this issue, we introduce a random \textit{dropout} procedure during training: an expert's prediction is randomly dropped with a probability of $p$ and the expert's weight is not trained on the input sample for which it is dropped.
This simple procedure helps reduce dependence on any single expert and ensures a relatively balanced load distribution.
%


\textbf{Additional Regularization. } As mentioned earlier, the experts can have individual costs associated with their consultation.
Let $C_{E_1, \dots, E_{m-1}}: \mathcal{X} \rightarrow \R^{m-1}$ refer to the vector of input specific cost of each expert consultation.
Assuming that the costs of the experts are independent of one another, we can take these costs into account in our framework by adding $ \lambda \cdot C_{E_1, \dots, E_{m-1}}(X)^\top D(X)_{-1}$ as a regularizer to the loss function, where $D(X)_{-1}$ denotes the first $(m-1)$ elements of the vector $D(X)$ and $\lambda > 0$ is a hyperparameter.


\section{Synthetic Simulations} \label{sec:synthetic_experiments}

We first test the efficacy of the joint learning framework and its variants on synthetic settings.
%
%
We use a synthetic and a real-world dataset for these simulations, and synthetically generate expert predictions for each input sample.
For all datasets, $L_{\text{clf}}$ will be the log-loss function and classifier $F$ will be the standard logistic function.

\subsection{Synthetic Dataset} \label{sec:synthetic_dataset}
%
%

\begin{figure}[t]
    \centering
    \includegraphics[width=0.5\linewidth]{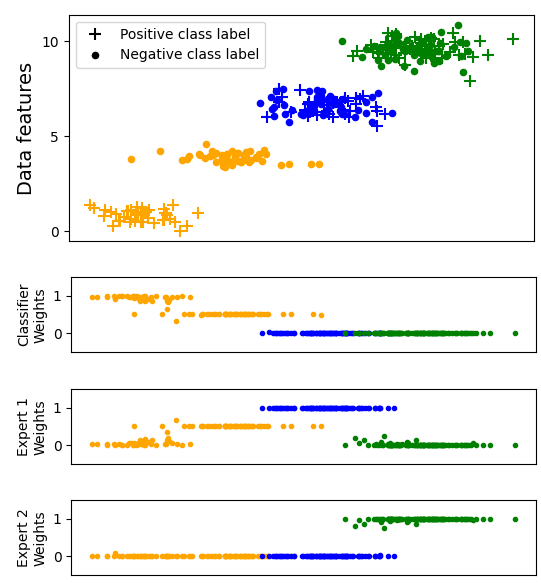}
    \caption{(\S\ref{sec:synthetic_dataset} simulations) The first plot shows the datapoints in the synthetic dataset. The next three plots show the weights assigned to classifier, expert 1 and expert 2 respectively for different clusters by the joint learning framework. 
    }
    \label{fig:syn_data_fig}
\end{figure}

\textbf{Dataset and Experts. }
Each sample in the dataset contains two features, sampled from a two-dimensional normal distribution, and a binary class label (positive or negative).
There are two available experts; their behaviour is described below.

Let $\mu \sim \text{Unif}(0,1)^2$ denote a randomly sampled mean vector and let $\Sigma \in \R^{2 \times 2}$ denote a covariance matrix that is a diagonal matrix with diagonal entries sampled from $\text{Unif}(0,1)$.
The data has 3 clusters, represented by colors \textit{orange}, \textit{blue}, and \textit{green}.
The \textit{orange} cluster has two further sub-clusters: the first sub-cluster is sampled from the distribution $\mathcal{N}(\mu, \Sigma)$ and is assigned class label 1, while the second sub-cluster is sampled from the distribution $\mathcal{N}(\mu + 3, \Sigma)$ and is assigned  label 0.
Since the sub-clusters are well-separated, this \textit{orange} cluster can be accurately classified using the two dimensions.

The \textit{blue} cluster is sampled from the distribution $\mathcal{N}(\mu+6, \Sigma)$, and each sample is assigned a class label 1 with probability 0.5.
Expert 1 is assumed to be accurate over the \textit{blue} cluster, i.e., if a sample belongs to the \textit{blue} cluster, expert 1 returns the correct label for that sample; otherwise it returns a random label.
Similarly, the \textit{green} cluster is sampled from the distribution $\mathcal{N}(\mu+9, \Sigma)$, each sample is assigned a class label 1 with probability 0.5 and Expert 2 is assumed to be accurate over the \textit{green} cluster and random for other clusters.

We construct a dataset with 1000 samples using the above process, with an almost equal proportion of samples in each cluster; the samples are randomly divided into train and test partitions (80-20 split).
The distribution of the data-points is graphically presented in \textbf{Figure~\ref{fig:syn_data_fig}}.
Suppose the hypothesis class of classifiers is limited to linear classifiers.
The ideal solution (in the absence of any expert costs) is for the classifier to accurately classify elements of the \textit{orange} cluster, and defer the samples from \textit{blue} cluster to expert 1 and the samples from \textit{green} cluster to expert 2.
If the linear classifier is learnt before training the deferrer, then it will try to reduce error across all clusters, and resulting framework will not be accurate over any cluster, since clusters \textit{blue} and \textit{green} cannot be linearly separated.
By studying the performance for this synthetic dataset we can determine if the joint learning framework accurately deciphers the underlying data-structure.

\vk{We also report the performance of two crowdsourcing algorithms: (a) \textit{LL} algorithm \cite{li2015cheaper} which tackles the worker selection problem, given the reliability and variance of all the workers, and (b) \textit{CrowdSelect} \cite{qiu2016crowdselect}, which aims to model the behaviour of the workers to appropriately allocate a subset of workers to each task.
For both crowdsourcing algorithms, the classifier is pre-trained using the train partition, and treated as just another worker.
The details of these algorithms are provided in Appendix~\ref{sec:baselines}.
}

\begin{table*}[t]
\small
\begin{tabular}{llccc}
\toprule
\multicolumn{2}{l}{Method} & \specialcell{Overall\\ Accuracy}  & \specialcell{Non-AAE\\ Accuracy}  & \specialcell{AAE\\ Accuracy} \\
\midrule
\multirow{3}{*}{Baselines} & Classifier only & .89 (.00) & .86 (.00) & .96 (.00) \\
& Randomly selected committee & .84 (.07) & .83 (.10) & .85 (.01) \\
& Randomly selected fair committee & .88 (.06) & .86 (.11) & .93 (.03) \\
& LL & .96 (.03) & .97 (.03) & .95 (.04) \\
& CrowdSelect & .91 (.04) & .89 (.06) & .93 (.04) \\
\midrule
\multirow{3}{*}{\specialcell{Joint learning frameworks\\and fair variants} } & Joint framework &  .92 (.02) & .89 (.03) & .97 (.00) \\
& Joint balanced framework & .94 (.01) & .92 (.02) & .98 (.00)\\
& Joint minimax-fair framework & .98 (.01) & .98 (.01) & .97 (.01)\\
\midrule
\multirow{3}{*}{\specialcell{Sparse variants of joint\\learning framework}}  & Joint sparse framework &  .92 (.01) & .90 (.02) & .96 (.01) \\
& Joint balanced and sparse framework & .92 (.01) & .89 (.01) & .97 (.00)\\
& Joint minimax-fair and sparse framework & .98 (.01) & .97 (.01) & .98 (.00)\\
\bottomrule
\end{tabular}
\caption{Overall and dialect-specific mean accuracies (standard error in brackets) for simulations in \S\ref{sec:offensive_synthetic}. 
}
\label{tab:main_results}
\end{table*}

\textbf{Implementation Details. }
We use projected gradient descent, with 3000 iterations, learning rate $\eta=0.05$, and $\alpha_1 = 0, \alpha_2 = 1$.
%
%
As discussed before, $\alpha_1 = 0$ can also train the classifier and deferrer simultaneously.

\textbf{Results. }
A baseline SVM classifier trained over the entire dataset has accuracy around 0.67 (accurate for one cluster and random over the other two).
In comparison, the joint learning framework has perfect (1.0) accuracy.
If the sparse variant of the joint learning framework is used with $k{=}1$ (defer to single expert), the accuracy drops to 0.91.
To better understand the performance of the framework, Figure~\ref{fig:syn_data_fig} presents the weights (normalized) assigned to the different experts (and classifier) for the test partition (bottom three plots).
%
%

Starting with the \textit{green} cluster, the lowest plot shows that expert 2 is assigned the highest weight for samples in this cluster, implying that the prediction for this cluster is always correctly deferred to expert 2.
Similarly, the prediction for the \textit{blue} cluster is always correctly deferred to expert 1.
For most of the samples in the \textit{orange} cluster, the weight assigned to the classifier is larger than the weights assigned to the two experts.
For some samples in this cluster, however, a non-trivial weight is also assigned to expert 1, which is why the accuracy for the sparse variant is lower than the accuracy of the non-sparse variant.
%
This can be prevented using non-zero expert costs, which we employ in the next simulation.

\vk{The baseline \textit{LL} algorithm achieves an accuracy of 67\% on this dataset; this is because it associates a single measure of aggregated reliability with each worker, which in this case is unsuitable since each worker has their specific domain of expertise.
The \textit{CrowdSelect} algorithm achieves the best accuracy of around 83\%; in this case, the error models for each expert and the classifier are constructed individually. Due to this, the algorithm is unable to perfectly stratify the input space amongst the experts (and classifier).
%
%
}

\textbf{Discussion. }
The purpose of this simulation was to show that the deferrer can choose experts in an input-specific manner.
The results show that the deferrer can indeed decipher the underlying structure of the dataset, and accordingly choose the expert(s) to defer to for each input (addressing the drawback of \textit{LL}).
The important aspect of the problem to notice here is that the cluster identity is the additional information available only to the experts.
The cluster identity is crucial for the experts as it reflects their domain of expertise and helps them make the correct prediction if the sample lies in their domain.
On the other hand, the cluster identity is useful to the deferrer only to defer correctly; even if the cluster is part of the input, the framework cannot use it to make a correct prediction, but can use it to defer to the correct expert.
In other words, the framework can use the available information to identify samples that need to be deferred to an expert (addressing the drawback of \textit{CrowdSelect}).
This sub-problem of directly identifying contentious input samples is also related to prior work by \citet{raghu2019direct}.

\subsection{Offensive Language Dataset} \label{sec:offensive_synthetic}

\textbf{Dataset.} Our base dataset consists of around 25k Twitter posts curated by \citet{davidson2017automated}; all posts are annotated with a label that corresponds to whether they contain hate speech, offensive language, or neither. We set class label to 1 if the post contains hate speech or offensive language, and 0 otherwise.
Using the dialect identification model of \citet{blodgett2017dataset}, we also label the dialect of the posts: African-American English (AAE) or not.
Around 36\% of the posts in the dataset labeled as AAE.
%
We treat dialect as the protected attribute in this case. %

%
%
%
%

\textbf{Experts. } The experts are constructed to be biased against one of the dialects.
%
%
We generate $m$ synthetic experts, with $\lfloor 3m/4 \rfloor $ experts biased against AAE dialect and $\lceil m/4 \rceil$ experts biased against non-AAE dialect. 
To simulate the first $\lfloor 3m/4 \rfloor$ experts, for each expert $i \in \set{1, \dots, \lfloor 3m/4\rfloor}$, we sample two quantities: $p_i{\sim}\text{Unif}(0.6, 1)$ and $q_i{\sim}\text{Unif}(0.6, p_i)$. For expert $i$, $p_i$ will be the its accuracy for the non-AAE group and $q_i$ will its accuracy for the AAE group.
To make a decision, if the input belongs to the non-AAE group then this expert outputs the correct label with probability $p_i$ and if the input belongs to the AAE group then this expert outputs the correct label with probability $q_i$.
%
By design, the first $\lfloor 3m/4 \rfloor $ experts can have a certain level of bias against the AAE group since $q_i < p_i$ for all $i \in \set{1, \dots, \lfloor 3m/4\rfloor}$.
%
The same process, with flipped $p_i$ and $q_i$, is repeated for the remaining $\lceil m/4 \rceil$ experts, so that they are biased against the non-AAE group.


\textbf{Baselines. }
There are three simple baselines that can be easily implemented: (1) using the classifier only, (2) randomly selected committee - a committee of size $\lceil m/4\rceil$ is randomly selected (in this case, the predictions are expected to be biased against the AAE dialect since most of the experts are biased against the AAE dialect - see \S\ref{sec:other_experiments}), and (3) random fair committee - i.e., if the post is in AAE dialect, the committee randomly selects from experts with higher accuracy for AAE group, and if the post is in non-AAE dialect, the committee randomly selects from experts with higher accuracy for non-AAE group. This committee selection should ensure relatively balanced accuracy across the dialects, and can therefore be used to judge the fairness of the joint learning framework. 
We also implement and report the performance of \textit{LL} and \textit{CrowdSelect} algorithms for this dataset.

\textbf{Implementation Details. }
The dataset is split into train and test partitions (80-20 split).
For both classifier and deferrer, we use a simple two-layer neural network, that takes as input a 100-dimensional vector corresponding to a given Twitter post (obtained using pre-trained GloVe embeddings \cite{pennington2014glove}).
The experts are given a cost of 1 each, i.e., $C_{E_1, \dots, E_{m-1}} = \mathbf{1}$ and  $\lambda = 0.05$ (the regularizer used is $\lambda \cdot \E[C_{E_1, \dots, E_{m-1}}(X)^\top D(X)_{-1}]$). Inspired by prior work on adaptive learning rate \cite{duchi2011adaptive}, exponent $c$ of parameter $\alpha$ is set at 0.5 and dropout rate at 0.2. 
We present the results for $m=20$ in this section and discuss the performance for different $m, \lambda$, and dropout rate in Appendix~\ref{sec:other_experiments}.
We use stochastic gradient descent for training with 
learning rate $\eta=0.1$
%
and for 100 iterations with batch size of 200 per iteration.
For the sparse variants with $m=20$, we sample $k=5$ experts from the output distribution.
The process is repeated 100 times, with a new set of experts sampled every time, and we report the mean and standard error of the overall and dialect-specific accuracies.

\textbf{Results.}
The results for the joint learning framework and its variants, along with the baselines are presented in \textbf{Table~\ref{tab:main_results}}.
The joint learning framework has a larger overall and group-specific average accuracy than the classifier.
%
%
The best group-specific and overall accuracy is achieved by the joint minimax-fair framework (and its sparse variant), showing that it is indeed desirable to enforce minimax-fairness in this setting as it leads to an overall improved performance across all groups.
%
%
The sparse variations of all joint frameworks, as expected, still have better performance than the classifier and random-selection baselines, and are quite similar to the non-sparse variants.
%
%
Joint fair (balanced and minimax-fair) frameworks also have similar or lower accuracy disparity across the groups than random fair committee baseline.
This shows that the learnt deferrer is also able to differentiate between biased and unbiased experts to an extent.
Due to the non-zero $\lambda$ parameter used, on average, the classifier is assigned around $5\%$ of the deferrer weight per input sample. This implies that, when creating sparse committees with $k=5$, the classifier is consulted for around 25\% of the input samples. This fraction can be further increased by appropriately increasing $\lambda$.

Further, due to our use of dropout, more accurate experts are not assigned disproportionately high weights, exhibiting the effectiveness of load balancing using dropout.
This is demonstrated in \textbf{Figure~\ref{fig:wts_plot}} in Appendix, which presents variation of the weights assigned by the joint framework to the experts vs the accuracies of the experts for a single repetition.
%
%

\vk{
%
The \textit{LL} algorithm is able to achieve very high overall accuracy ($\geq 95\%$ for both groups) for this setting. However, our joint minimax-fair sparse framework has two advantages over \textit{LL} algorithm. First, it achieves relatively better accuracy for both dialect groups. Second, \textit{LL} pre-selects the most accurate experts to whom all the inputs are deferred. This is problematic and inefficient since \textit{LL} only uses $k$ out of $m$ experts; in comparison, our algorithm distributes the input samples amongst all experts to reduce the load on the most accurate experts (see Figure~\ref{fig:wts_plot} in Appendix).
\textit{CrowdSelect}, on the other hand, achieves lower overall and group-specific accuracies than joint minimax-fair frameworks.
}

\section{Simulations Using Real-world Data For The Offensive Language Dataset} \label{sec:mturk}


The simulations in the previous sections highlighted the effectiveness of the joint learning framework in improving the accuracy and fairness of the final prediction.
In this section, we present the results on a similar real-world dataset of Twitter posts, annotated using Mechanical Turk (MTurk).
%


\textbf{Dataset.} We use a dataset of 1471 Twitter posts for the MTurk survey. This is a subset of the larger dataset by \citet{davidson2017automated}.
Importantly, this dataset is jointly balanced across the class categories used in \citet{davidson2017automated} and the two dialect groups (as predicted using \citet{blodgett2017dataset}).
Once again, the labels from \citet{davidson2017automated} are treated as the \textit{gold labels} for this dataset.
%
%




\textbf{MTurk Experiment Design. }
%
The MTurk survey presented to each participant started with an optional demographic survey.
%
This was followed by 50 questions; each question contained a Twitter post from the dataset and asked the participant to choose one of the following options: `Post contains threats or insults to a certain group', `Post contains threats or insults to an individual', `Post contains other kinds of threats or insults, such as to an organization or event',  `Post contains profanity', `Post does not contain threats, insults, or profanity'. 
The options presented to the user are along the lines of the taxonomy of offensive speech suggested by \citet{zampieri2019predicting}.
The first four options correspond to offensive language in the Twitter post, while the last option corresponds to the post being non-offensive.
%
As in the synthetic simulations, the participants are also provided the predicted dialect label of the post.
The participants were paid a sum of \$4 for completing the survey (at an hourly rate of \$16).

%

\begin{table}
\small
\centering
\begin{tabular}{lccc}
\toprule
Method & \specialcell{Overall\\Accuracy}  & \specialcell{Non-AAE\\Accuracy}  & \specialcell{AAE\\Accuracy} \\
\midrule
Classifier only & .78 (.02) & .76 (.05) & .80 (.04) \\
Joint framework &  .85 (.03) & .87 (.04) & .83 (.03) \\
Joint balanced framework & .84 (.03) & .87 (.03) & .81 (.04)\\
Joint minimax framework & .85 (.02) & .87 (.02) & .83 (.02) \\
\bottomrule
\end{tabular}
\caption{Results of the joint learning framework and fair variants on the MTurk dataset.}
\label{tab:results_joint_mturk}
\end{table}

\textbf{MTurk Experiment Results. }
%
Overall, 170 MTurk workers participated in the survey and each post in the dataset was labeled by around 10 different annotators.
Since each participant only labels a fraction of the dataset, we will treat this setting as one where there are missing expert predictions during the training of the joint learning framework.
%
The inter-rater agreement, as measured using Krippendorff's $\alpha$ measure, is 0.27.
As per heuristic interpretation \cite{gwetkrippendorff}, this level of interrater agreement is considered quite low for a standard dataset annotation task.
However, it is suitable for our purpose since our framework aims to address situations where there is considerable disparity in the performances of different humans in the pipeline, and the goal of the joint learning framework is to choose the annotators that are expected to be accurate for the given input.


The overall accuracy of the aggregated responses (i.e., taking a majority of all responses for every post and comparing to the \textit{gold label}) is around 87\%, which is close to the accuracy of the automated classifier in \S\ref{sec:offensive_synthetic} (84\% for AAE posts and 91\% for non-AAE posts).
%
%
The high accuracy shows that using crowdsourced annotations in this setting is quite effective and the hypothetical \textit{aggregated crowd annotator} can indeed be considered an \textit{expert} for this content moderation task.
However, the individual accuracies of the experts is arguably more interesting and relevant to our setting.


The average individual accuracy of a participant is 77\% ($\pm 13\%$). The minimum individual accuracy is $\approx$ 38\% while the maximum individual accuracy is 98\%.
The wide range of accuracies evidences large variation in annotator expertise for this task.
The individual accuracies for posts from different dialects also presents a similar picture.
The average individual accuracy of a participant for the AAE dialect posts is 76\% ($\pm 15\%$) and average individual accuracy of a participant for the non-AAE dialect posts is 78\% ($\pm 14\%$).


While mean individual accuracies for the two dialects are quite similar, most annotators do display a disparity in their accuracy across the two groups.
92 of the 170 participants had a higher accuracy when labeling posts written in a non-AAE dialect. The average difference between the accuracy for non-AAE dialect posts and AAE dialect posts for this group of participants was 8.5\% ($\pm 6.6\%$). 
75 participants had a higher accuracy when labeling posts written in the AAE dialect. The average difference between the accuracy for AAE dialect posts and non-AAE dialect posts was 7.1\% ($\pm 5.5\%$). 
Three remaining participants were equally accurate for both groups.
The disparate accuracies here are quite similar to those in the early 
synthetic simulations.
We next analyze the performance of joint learning framework on this dataset.

\textbf{Joint Learning Framework Results on MTurk Dataset. }
We perform five-fold cross validation on the collected dataset. For each fold, we train our joint learning framework (with $\eta=0.3$) on the train split and evaluate it on the test split.
%
%
Since expert decisions are available only for a subset of the dataset, we do not use dropout or expert costs.
Results are shown in \textbf{Table~\ref{tab:results_joint_mturk}}. 
As before, the overall accuracy of the joint learning frameworks is higher than the accuracy of the classifier alone.
Amongst the fair variants, even though the accuracy for both dialect groups is larger when using the balanced or minimax loss function (compared to the classifier alone), it does not lead to significantly different group-specific accuracies vs.\ simple joint learning framework.
The performance of sparse variants is presented in Appendix~\ref{sec:mturk_other}.
Since a relatively small number of prior predictions is available for each expert, the task of differentiating between experts here is tougher. Hence, sparse variants perform similar or better than the classifier when committee size $k$ is around 60 or greater.

\textbf{Discussion. }
%
%
%
The wide range in accuracy observed across annotators confirms the expectation that different humans-in-the-loop will naturally bring varying levels and domains of expertise.
Their accuracy will be affected by not only the training they receive, but also by their background. For example, native speakers of a given dialect are naturally expected to be better annotators for language examples from that dialect.
However, despite the difficulty of the task and the disparity in group accuracies, our joint learning framework is still able to identify the combination of experts that are suitable for any given input and, correspondingly, increase the accuracy and fairness of the final prediction.






\section{Discussion, Limitations, and Future Work} \label{sec:limitations}
Our proposed framework addresses settings that involve active human-machine collaboration. 
Having shown its efficacy for synthetic and real-world datasets, 
we next highlight certain limitations and fruitful directions for future work.
%

\textbf{Fairness of the Framework. }
It is crucial that the framework is fair with respect to the protected attribute. We proposed two methods for ensuring that the predictions are unbiased: by trying to achieve a balanced error rate for all groups, or by trying to minimize the maximum group-specific error rate (minimax Pareto fairness).
Both fairness mechanisms can handle multi-class protected attributes, which helps generalize our framework to settings beyond simple binary protected attributes (e.g., multiple racial categories).
%
%
%
An additional advantage of using these fairness definitions is that the protected group labels are not required for test or future samples, addressing the issue of their possible unavailability due to policy or privacy restrictions \cite{edwards2017slave}.

As mentioned in \S\ref{sec:variants}, other fairness mechanisms can also be incorporated into our framework. 
%
%
%
For most applications, the choice of fairness mechanism and constraint is often a context-dependent question.
An uninformed choice of these variables can possibly lead to a degradation of both accuracy and fairness \cite{liu2018delayed} and, therefore, it is important to take the impact of any fairness constraint on the user population into account before its implementation.
Similarly, in our setting, it is important to first decide whether the goal of fairness is minimizing the worst group error, demographic parity, etc., and then choose the mechanism to implement it.
%

\textbf{Real-world Benchmark Dataset. }
We created an MTurk dataset for evaluating human-in-the-loop prediction frameworks with multiple experts for detecting hate speech.
The goal of constructing this dataset
was to facilitate the learning and evaluation of hybrid frameworks, since having a large number of annotations for each input better enables a learning procedure to differentiate between annotators with different abilities.
Existing datasets have often released only aggregate labels, such as by majority voting, which supports ML model training but 
%
%
%
does not allow modeling individual annotators.
%
To be able to release such data, we have replaced annotator platform IDs with automatically generated pseudonyms.

Our new dataset has important limitations.
First, in order to obtain a large number of annotations for each Twitter post, we kept the dataset size relatively small.
Furthermore, since the dataset is a subset of the dataset constructed by \citet{davidson2017automated}, it cannot be considered representative of the larger population of Twitter posts/users and the performance demonstrated in our simulations may not translate to larger Twitter datasets.
The number of human annotators (170) in our survey is also larger than desired, even though each annotator labels 50-100 posts.
Our framework aims to learn the domain of expertise of the human experts using only the prior decisions of the experts.
However, it is not completely clear how many prior decisions are needed to accurately determine the domain of expertise of every annotator.
%
The gap between the performance using synthetic experts (\S\ref{sec:synthetic_experiments}) and real-world experts (\S\ref{sec:mturk}) partially shows that it might be necessary to get more predictions for each expert.

\citet{forough2020}, in a position paper on human-in-the-loop frameworks in facial recognition, argue the necessity of real-world empirical studies of such frameworks to justify their widespread use. They also list the technical challenges associated with such empirical studies.
The real-world dataset we provide attempts to initiate a real-world empirical study of human-in-the-loop frameworks for content moderation but, at the same time, faces similar challenges as highlighted by \citeauthor{forough2020}, i.e., issues with data availability and generalizability of participants/context.

%
%
%
%

\textbf{MTurk Experiment Generalizability.}
Similar to any other study done using MTurk participants, questions can be raised about the generalizability of the results to a larger population.
%
%
%
While MTurk participants do seem suitable for detecting offensive language in Twitter posts (as seen from the performance of the \textit{aggregated crowdworker} in \S\ref{sec:mturk}), they may not accurately represent how a lay person would respond to a similar survey or how a domain expert would judge the same posts.
The performance of domain experts (people with more experience in screening offensive language) will most likely be better than the accuracy of an average crowd annotator. Correspondingly, our framework with better trained content moderation experts can be expected to have similar or better performance.
%
Nevertheless, as pointed out in prior work \cite{forough2020, alba2017multi}, experimental design and choice of participants will play a much bigger role in simulating human-in-the-loop frameworks in settings where human experts cannot be imitated by volunteers.

\textbf{Replaceable Experts.}
%
%
An extension of our model that can be further explored is addition/removal of experts.
%
%
%
If a new expert is added to the pipeline and the domain of expertise of this expert is different than the domain of the replaced/existing experts, then the framework 
might need to be retrained to appropriately include the new expert.
%
This overhead of retraining can, however, be avoided.
%
For instance, one could train the framework using a \textit{basis of experts}, i.e., divide the feature space into interpretable sub-domains and map the experts to these sub-domains.
Then if we train the framework using sample decisions of experts with disjoint sub-domains of expertise, we can ensure that the entire feature space is covered either by the classifier or the deferrer (in a similar manner as \S\ref{sec:synthetic_dataset}), and any new expert could be mapped to the corresponding sub-domain.
%
Approaches from prior work \cite{strobl2009introduction,lopez2019ensemble} can be potentially used to learn these sub-domains and extend our joint learning framework for such settings.


\textbf{Improved Implementation.} \label{sec:impl_issues}
Like other complex frameworks involving many decision making components, our framework can also suffer from issues that arise from real-world implementations. 
%
%
%
For instance, dropout reduces overdependence on any particular expert, but does not consider the load on any small subset of experts.
Alternate load distribution techniques (e.g., \citet{nguyen2015combining}) can be explored further, at the risk of inducing larger committee sizes. 
%
%
%
Another extension that can be pursued is 
to keep the committee size small but variable; this can help with load distribution as well as better committee selection.
%

\section{Conclusion}
We proposed a joint learning model to simultaneously train a classifier and a deferrer in the multiple-experts setting. 
%
%
%
%
{\bf The code and dataset are available online}\footnote{\url{https://github.com/vijaykeswani/Deferral-To-Multiple-Experts}}.
Our framework can help increase the applicability of automated models in settings where human experts are an indispensable part of the pipeline. 
At the same time, by addressing the domains and biases of the model and the humans, we ensure that its utilization is thoughtful and context-aware.
%

\section*{Acknowledgments}

We thank the many talented Amazon Mechanical Turk workers who contributed to our study and made this work possible. We also thank our reviewers for their valuable feedback. The statements made herein are solely the opinions of the authors.

\bibliographystyle{plainnat}
\bibliography{references}

\clearpage
\appendix

\begin{figure}
\centering
 \includegraphics[width=0.7\linewidth]{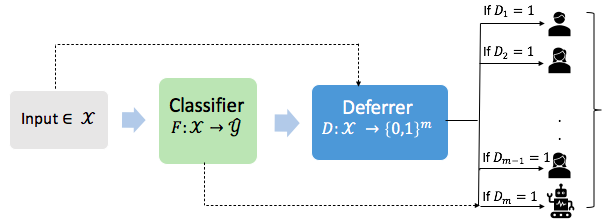}
 \caption{Overview of our model.}
 \label{fig:model}
 \end{figure}

\section{Proofs} \label{sec:proofs}
In this section, we present the proofs of the theoretical statements made in the main text.

\subsection{Proof of Theorem~\ref{thm:gradient}}

\begin{proof}
Both the propositions can be proved together.
Note that 
\[\sigma'(x) = \frac{2e^xe^{1-x}}{(e^x + e^{1-x})^2}.\]
Therefore,
\begin{align*}
\pdv{L}{D} = -Y \cdot \frac{2e^{1-D^\top Y_{E,1}}}{e^{D^\top Y_{E,1}} + e^{1-D^\top Y_{E,1}}} \cdot Y_{E,1}+ (1-Y) \cdot \frac{2e^{D^\top Y_{E,1}}}{e^{D^\top Y_{E,1}} + e^{1-D^\top Y_{E,1}}} \cdot Y_{E,1}, 
\end{align*}
which leads to the statement of the theorem.
%

\end{proof}

\subsection{Proof of Proposition~\ref{prop:convex}}
\begin{proof}
Convexity with respect to $D$ can be shown as an extension of the above proof.
Taking the second derivative with respect to $D$ also shows that it is always non-negative, implying that $L$ is convex with respect to $D$.
Similarly, the first part of $L$ is convex in $F$ (since $L_{\text{clf}}$ is convex) and the second part contains the negative log-exponent of the product of $F$ and the last coordinate of $D$, and hence is convex in $F$ as well.
\end{proof}

\subsection{Proof of Theorem~\ref{thm:convergence}}
In this section, we prove the convergence bound presented in Theorem~\ref{thm:convergence}.
By using the standard projected gradient-descent convergence bound stated below, we can prove the convergence rate bound for projected gradient-descent algorithm on the joint learning model.

\begin{theorem}[\cite{hiriart2013convex, boyd2004convex}] \label{thm:grad_desc_bound}
Given a convex, $\ell$-Lipschitz smooth function $f: \R^n \rightarrow \R$, oracle access to its gradient, starting point $x_0 \in \R^n$ with $\norm{x_0 - x^\star} \leq \delta$ (where $x^\star$ is an optimal solution to $\min_{x} f(x)$) and $\varepsilon > 0$, the projected gradient descent algorithm, with starting point $x_0$, step-size $1/2\ell$ and after $T$ iterations, returns a point $x$ such that
\[f(x) \leq f(x^\star) + \varepsilon.\]
Here $T = O \left( \frac{\ell\delta^2}{\varepsilon} \right)$. 
\end{theorem}

\begin{proof}[Proof of Theorem~\ref{thm:convergence}]

We have the following loss function:
\begin{align*}
\textstyle L(F,D)= \alpha_1\E_{X, Y} \left[ L_{\text{clf}}(F; X,Y)\right] - \alpha_2 \E_{X, Y} \left[ Y \log(\hat{Y}_D) + (1-Y)\log (1 - \hat{Y}_D)\right].        
\end{align*}


The first step is to find the upper bound on the Lipschtiz-smoothness of combined loss function.
To that end, we first calculate the Lipschtiz-smoothness constants of $L$ with respect to $F$ and $D$ individually.
By definition,
\[\pdv[2]{L_{\text{clf}}}{F} \preccurlyeq \ell I.\]

Let $L_D := \E_{X, Y} \left[ Y \log(\hat{Y}_D) + (1-Y)\log (1 - \hat{Y}_D)\right]$.
Then,
\[\pdv{\hat{Y}_D}{F} = 2\hat{Y}_D^2 e^{1-2D^\top Y_E} D^{(m)}.\]
Using the above derivative, we get that
\[\pdv[2]{L_D}{F} \preccurlyeq 8e^2 \ell I.\]
Therefore, 
\[\pdv[2]{L}{F} \preccurlyeq (\alpha_1 + \alpha_28e^2) \ell I.\]

For the Lipschtiz-smoothness of $L$ with respect to $D$, note that we can use results on Lipschitz-smoothness of logistic-regression (since $L_D$ corresponds to log-loss with logistic regression parameter $D$).  
In particular, 
\[\pdv[2]{L}{D} \preccurlyeq 2 \alpha_2\max \text{eig}(Y_E^\top Y_E) I \preccurlyeq 2\alpha_2 mI.\]
The second inequality follows from the fact that matrix $Y_E$ only contains 0-1 entries.
For the cross second-derivative, from proof of Theorem~\ref{thm:gradient} we have that
\begin{align*}
\pdv{L}{D} =& -2 \alpha_2 Y \cdot (1-\sigma(D^\top Y_{E,1}))  \cdot Y_{E,1}\\ &+ 2 \alpha_2 (1-Y) \cdot \sigma(D^\top Y_{E,1}) \cdot Y_{E,1}.
\end{align*}
Therefore,
\begin{align*}
\pdv{L}{D}{F} =  2  \alpha_2\sigma(D^\top Y_{E,1})^2 e^{1-2D^\top Y_{E,1}} \cdot Y_{E,1} D_m.
\end{align*}
We simply need to bound the Frobenius norm of the above second derivative operator for our setting.
\[\left\lVert\pdv{L}{D}{F}\right\rVert_F \preccurlyeq 2\alpha_2e \sqrt{m}.\]

Therefore, combining the above inequalities, we get that the joint Lipschtiz-smoothness constant of $L$ with respect to $(F,D)$ (given constant $\alpha_1, \alpha_2$) is $\ell'$, where
\[\ell' \leq c (\ell + m),\]
where $c > 0$ is a constant.
Next, since we are using projected gradient descent algorithm, we know that the $\norm{D}^2 \leq m$.
Therefore, applying Theorem~\ref{thm:grad_desc_bound}, we get that we can converge to $\varepsilon$-close to the optimal solution using step-size $O((\ell + m)^{-1})$ and $T$ iterations, where
\[T = O\left( \frac{ (\ell + m)(\delta^2 + m)}{\varepsilon} \right).\]




\end{proof}

\subsection{Proof of Theorem~\ref{thm:minimax_fair}}

\begin{proof}
We will first prove the theorem when\\ $\hat{z} := \arg \max_{z \in \mathcal{Z}} L^z(F^\star, D^\star) =0$, i.e.,
\[L^1(F^\star, D^\star) \leq L^0(F^\star, D^\star) \leq \max_{z \in \mathcal{Z}} L^z(F^\circ, D^\circ).\]
Let $\beta = \mathbb{P}[Z = \hat{z}]$.
Then for any $F,D$,
\[L(F,D) = \beta \cdot L^0(F,D) + (1-\beta) \cdot L^1(F,D),\]
and by definition,
\[L(F^\star, D^\star) \geq L(F^\circ, D^\circ).\]
We will further divide the analysis into two cases. Case 1:
\[L^1(F^\circ, D^\circ) \leq L^0(F^\circ, D^\circ),\]
By definition of minimax-fair solution then,
\[L^0(F^\star, D^\star) \leq L^0(F^\circ, D^\circ).\]
Next we use this inequality to look at $L^1(F^\star, D^\star)$.  
\begin{align*}
    &L(F^\star, D^\star) \geq L(F^\circ, D^\circ) \\
    \Rightarrow &\beta \cdot L^0(F^\star, D^\star) + (1-\beta) \cdot L^1(F^\star, D^\star) \geq \beta \cdot L^0(F^\circ, D^\circ) + (1-\beta) \cdot L^1(F^\circ, D^\circ) \\
    \Rightarrow & \beta \cdot L^0(F^\star, D^\star) + (1-\beta)\cdot L^1(F^\star, D^\star) \geq \beta\cdot  L^0(F^\star, D^\star) + (1-\beta)\cdot L^1(F^\circ, D^\circ) \\
    \Rightarrow  &L^1(F^\star, D^\star) \geq L^1(F^\circ, D^\circ).
\end{align*}
Therefore, the risk disparity in this case
\begin{align*}
    \abs{L^0(F^\star, D^\star) - L^1(F^\star, D^\star)} 
     &= L^0(F^\star, D^\star) - L^1(F^\star, D^\star) \\
    & \leq L^0(F^\circ, D^\circ) - L^1(F^\circ, D^\circ).
\end{align*}
Hence the theorem is true in this case. 

Case 2:
\[L^0(F^\circ, D^\circ) \leq L^1(F^\circ, D^\circ),\]
By definition of minimax-fair solution then,
\[L^0(F^\star, D^\star) \leq L^1(F^\circ, D^\circ).\]
Once again we use this inequality to look at $L^1(F^\star, D^\star)$.  
\begin{align*}
    &L(F^\star, D^\star) \geq L(F^\circ, D^\circ) \\
    \Rightarrow &\beta \cdot L^0(F^\star, D^\star) + (1-\beta) \cdot L^1(F^\star, D^\star) \geq \beta \cdot L^0(F^\circ, D^\circ) + (1-\beta) \cdot L^1(F^\circ, D^\circ) \\
    \Rightarrow &\beta \cdot L^0(F^\star, D^\star) + (1-\beta) \cdot L^1(F^\star, D^\star) \geq \beta \cdot L^0(F^\circ, D^\circ) + (1-\beta) \cdot L^0(F^\star, D^\star) \\
    \Rightarrow & (1-\beta) \cdot L^1(F^\star, D^\star)\geq \beta \cdot L^0(F^\circ, D^\circ) + (1-2\beta) \cdot L^0(F^\star, D^\star) \\
    \Rightarrow & (1-\beta) \cdot L^1(F^\star, D^\star)\geq \beta \cdot L^0(F^\circ, D^\circ) + (1-2\beta) \cdot L^1(F^\star, D^\star) \\
    \Rightarrow  &L^1(F^\star, D^\star) \geq L^0(F^\circ, D^\circ).
\end{align*}
Therefore, the risk disparity in this case
\begin{align*}
    \abs{L^0(F^\star, D^\star) - L^1(F^\star, D^\star)} 
     &= L^0(F^\star, D^\star) - L^1(F^\star, D^\star) \\
     \leq & L^1(F^\circ, D^\circ) - L^0(F^\circ, D^\circ) \\
    = &\abs{L^0(F^\circ, D^\circ) - L^1(F^\circ, D^\circ)} \\
\end{align*}
Hence the theorem is true in this case as well. 
The proof for $\hat{z} := \arg \max_{z \in \mathcal{Z}} L^z(F^\star, D^\star) =1$ follows by symmetry.
\end{proof}

\subsection{Proof of Theorem~\ref{thm:sparsity}}

\begin{proof}
Recall that in the binary class setting, given deferrer output $D$ and expert predictions $Y_E$ the output probabilistic prediction is calculated as $$\hat{Y}_D := \sigma \left(D^\top Y_{E}\right).$$
For simplicity of presentation, since we are talking about a single input setting we are removing the input $X, W$ in the formulas, i.e., $D(X)$ is represented as just $D$ and $E_i(X,W)$ is just $E_i$.
Let $E_{r_1}, \dots, E_{r_k}$ denote the $k$ experts sampled according to the distribution induced by $D(X)$. Then the output of the sparse framework is
$$\tilde{Y}_{D,k} := \sigma \left(\sum_{i=1}^m D^{(i)} \cdot \frac{1}{k} \sum_{i=1}^k E_{r_i}\right).$$
First we look at $\hat{Y}_D$.
\begin{align*}
    \log \hat{Y}_D  = D^\top Y_{E} - \log (e^{D^\top Y_{E}} + e^{1 - D^\top Y_{E}})
\end{align*}
Similarly,
\begin{align*}
    \log \tilde{Y}_{D,k}   = \sum_{i=1}^m D^{(i)} \cdot \frac{1}{k} \sum_{i=1}^k E_{r_i}- \log (e^{\sum_{i=1}^m D^{(i)} \cdot \frac{1}{k} \sum_{i=1}^k E_{r_i}} + e^{1 - \sum_{i=1}^m D^{(i)} \cdot \frac{1}{k} \sum_{i=1}^k E_{r_i}})
\end{align*}
Let $N(D):= \log (e^{D^\top Y_{E}} + e^{1 - D^\top Y_{E}})$ and let 
\[N'(D,k) := \log (e^{\sum_{i=1}^m D^{(i)} \frac{1}{k} \sum_{i=1}^k E_{r_i}}{+}e^{1 - \sum_{i=1}^m D^{(i)} \frac{1}{k} \sum_{i=1}^k E_{r_i}}).\]
Then, taking the absolute difference of log-losses, we get
\begin{align*}
    \E \abs{\log\hat{Y}_D - \log \tilde{Y}_{D,k}} &\leq \E\abs{D^\top Y_{E} - \sum_{i=1}^m D^{(i)} \cdot \frac{1}{k} \sum_{i=1}^k E_{r_i}}\\ &+ \E\abs{N(D)- N'(D,k)}.
\end{align*}
We will analyze the two terms separately. Note that for an expert sampled from distribution induced by $D$, we have that
\[\mathbb{E}_{r \sim D}[E_r] \cdot \sum_{i=1}^m D^{(i)}  = D^\top Y_E.\]
Therefore,
\begin{align*}
\E\abs{D^\top Y_{E} - \sum_{i=1}^m D^{(i)} \cdot \frac{1}{k} \sum_{i=1}^k E_{r_i}} &= \sum_{i=1}^m D^{(i)} \cdot \E\abs{ \frac{1}{k} \sum_{i=1}^k \mathbb{E}_{r \sim D}[E_r] - E_{r_i} } \\
& \leq \sum_{i=1}^m D^{(i)} \cdot \frac{1}{k} \sum_{i=1}^k \E\abs{ \mathbb{E}_{r \sim D}[E_r] - E_{r_i} } =  \sum_{i=1}^m D^{(i)} \cdot s_D,
\end{align*}
where $s_D$ represents the mean absolute deviation with respect to distribution induced by $D$.
For the second absolute difference, note that both
\[D^\top Y_{E}, \sum_{i=1}^m D^{(i)} \cdot \frac{1}{k} \sum_{i=1}^k E_{r_i} \leq \sum_{i=1}^m D^{(i)}.\]
When $x > 0$,
\begin{align*}
\log (e^x + e^{1-x}) &= \log (e^{-x} (e^{2x} + e))\\ &\leq  \log (e^{2x} + e) \\ &\leq \log 2 + \max(2x, 1).
\end{align*}
Furthermore, $\log (e^x + e^{1-x})$ is convex and achieves minimum value $0.5 + \log 2$.
Therefore, using the above upper and lower bounds, we get
\[\E\abs{N(D)- N'(D,k)} \leq \max\left(2\sum_{i=1}^m D^{(i)}, 1\right) -0.5.\]
Hence, 
\[\E \abs{\log\hat{Y}_D{-}\log \tilde{Y}_{D,k}} < s_D \norm{D}_1  + \max\left(2\norm{D}_1, 1\right).\]
\end{proof}

\begin{figure}[t]
    \centering
    \includegraphics[width=0.5\linewidth]{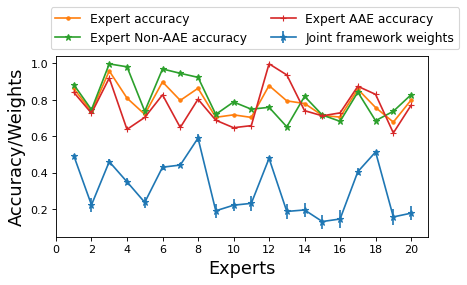}
    \caption{Weights assigned by the joint learning model and the accuracies of the 20 experts (one iteration shown). Accuracies and weights are seen to follow a similar pattern.}
    \label{fig:wts_plot}
\end{figure}

\begin{figure*}[t]
\centering
\includegraphics[width=\linewidth]{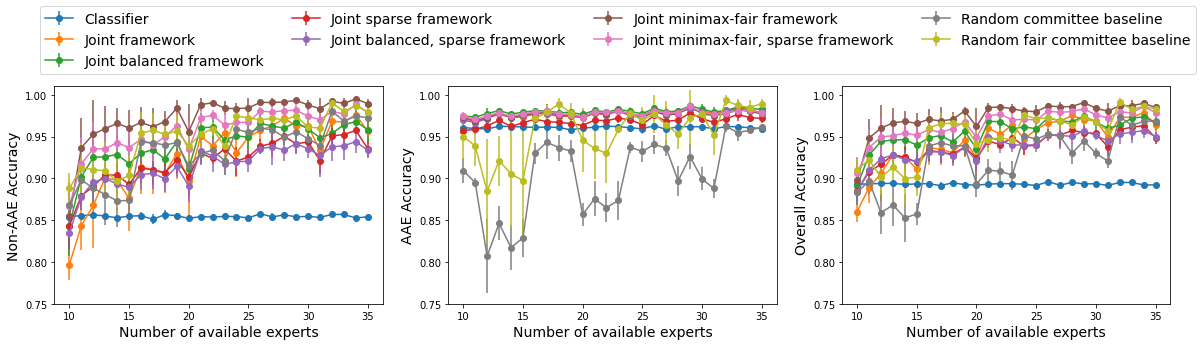} 
\caption{{Performance of all methods for different number of available experts.}}
\label{fig:acc_by_diff_experts}
\end{figure*}	

\begin{figure*}[t]
\centering
\includegraphics[width=\linewidth]{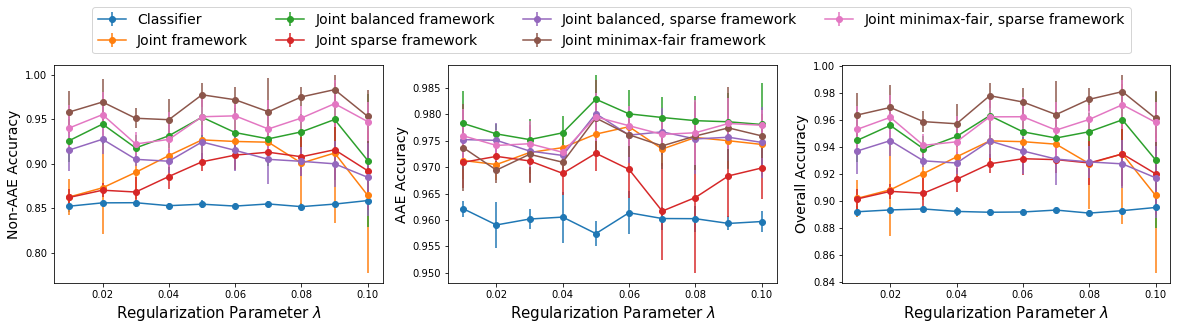} 
\caption{{Performance of all methods for different values of regularization parameter $\lambda$.}}
\label{fig:acc_by_diff_lambda}
\end{figure*}	

\begin{figure*}[t]
\centering
\includegraphics[width=\linewidth]{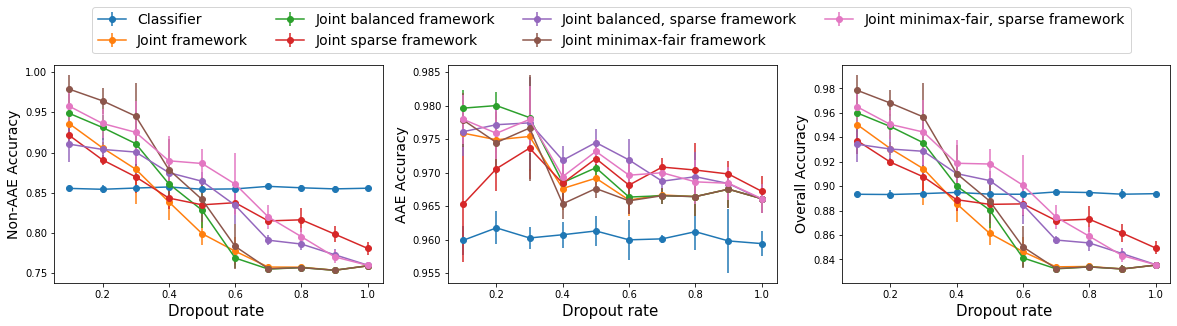} 
\caption{{Performance of all methods for different dropout rates.}}
\label{fig:acc_by_diff_dropout}
\end{figure*}

\begin{figure*}[t]
\centering
\includegraphics[width=\linewidth]{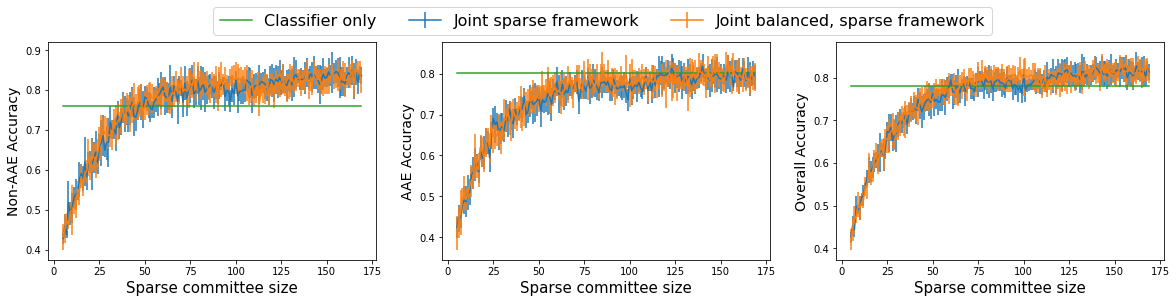} 
\caption{{Performance of sparse variants of the joint frameworks on the MTurk dataset for different committee sizes $k$.}}
\label{fig:sparse_mturk}
\end{figure*}	

\section{Details of baselines} \label{sec:baselines}

\subsection{LL Algorithm \cite{li2015cheaper}}
This algorithm, proposed by \citet{li2015cheaper}, takes as input the a single measure of reliability for each expert and returns $k$ experts using a formula that takes into account the reliability, number of classes and size of desired committee size $k$ (see Algorithm 1 in \cite{li2015cheaper}).
To calculate the measure of reliability for each expert using the training set, we simply calculate the accuracy of each expert over the training set.

Note that the main drawbacks of this approach is that it simply returns a single committee, i.e. does not choose the experts in an input-specific manner, and that it treats the pre-trained classifier as yet another learner. 

\subsection{CrowdSelect Algorithm \cite{qiu2016crowdselect}}
CrowdSelect is a more advanced task-allocation algorithm that takes into account the error models of different experts, as well as, their task-specific reliabilities and the individual costs associated with each expert consultation.
However, the proposed algorithm assumes that \textit{error rate of workers} for any given task is provided as input or can be estimated using autoregressive methods that use the task identities.
In our setting, the specific task classification (for example, cluster identity in case of Section~\ref{sec:synthetic_dataset})) may not be available; hence, these error models need to be separately constructed.

To construct the error models for the experts, for each expert $i$, we simply train a two-layer neural network $h_i$ on the train feature vectors using binary class labels that correspond to whether the expert's prediction for the given train feature was correct or not.
Then, for any test/future sample, $h_i$ will return the probability that the expert $i$ returns a correct prediction.
Using these error models, we then implement Algorithm 1 in \cite{qiu2016crowdselect} to get input-specific committees.

There are three drawbacks with this approach: (1) the pre-trained classifier is once again treated as yet another learner, (2) it is only applicable for binary classification (\cite{qiu2016crowdselect} propose studying extensions to non-binary as future work), and (3) the error models of all experts are learnt independently - this is inhibitory since it does not allow the perfect stratification of input domain into the domains of different experts.

Our method addresses all three drawbacks by learning a single deferrer and learning it simultaneously with a classifier.

\section{Other empirical results for offensive language dataset with synthetic experts} \label{sec:other_experiments}
In this section, we present additional empirical results for the offensive language dataset with multiple synthetic experts.

\subsection{Variation with number of experts}
We vary the number of experts $m$ from 10 to 35, while keeping $\lambda$ fixed at 0 and dropout rate fixed at 0.2, and present the variation of overall and dialect-specific accuracies when using different number of experts. The other parameters are kept to be the same as \S\ref{sec:synthetic_experiments}.
The results are presented in \textbf{Figure~\ref{fig:acc_by_diff_experts}}.
As expected, the performance of all joint frameworks increases with increasing number of experts, and the performance of minimax-fair framework is better than other methods in most cases.

\subsection{Performance of random committee baselines}
Figure~\ref{fig:acc_by_diff_experts} also provides further insight into the random committee baselines.
Since $75\%$ of the experts are biased against the AAE dialect, simply choosing the committee randomly leads to reduced accuracy for the AAE dialect.
When the committee is selected in a dialect-specific manner (random fair committee baseline), the disparity across dialects reduce but the accuracies of the experts are not taken into account.
The performance of these two baselines highlight the importance of selecting the experts in an input-specific manner and taking accuracies/biases of experts while deferring.

\subsection{Impact of $\lambda$}

We next vary the parameter $\lambda$ from 0.01 to 0.1, while keeping number of experts at $m=20$ and dropout rate at 0.2, and present the variation of overall and dialect-specific accuracies for different $\lambda$ values. 
The results are presented in \textbf{Figure~\ref{fig:acc_by_diff_lambda}}.
The variation with respect to $\lambda$ shows that setting its value close to 0.05 leads to best performance for most methods.
Smaller values of $\lambda$ will lead to low dependence on the classifier, while higher values of $\lambda$ implies associating larger regularization costs with the experts, and the figure shows that the performance for large $\lambda$ has larger variance and/or is closer to the performance of the classifier.

\subsection{Impact of dropout rate}
Finally we vary the dropout rate from 0.1 to 0.9, while keeping number of experts at $m=20$ and $\lambda=0.05$, and present the variation of overall and dialect-specific accuracies for different dropout rates. 
As expected, larger values of dropout can imply that the framework is unable to decipher the accuracies of the experts and, hence, leads to a drop in accuracy.
Reasonable levels of dropout rate (around 0.2), on the other hand, do not impact accuracy but significantly reduce load on the more accurate experts.

\section{Other empirical results for MTurk dataset} \label{sec:mturk_other}
As mentioned in \S\ref{sec:mturk}, the task of differentiating between the experts is more challenging for the MTurk dataset since relatively fewer prior predictions are available for each expert.
Correspondingly, the sparse variants do not perform so well when the chosen committee size $k$ is small.

The performance of the sparse variants, as a function of $k$, is presented in \textbf{Figure~\ref{fig:sparse_mturk}}.
From the figure, one can see that to achieve performance similar or better than the classifier, $k$ needs to be around 60 or larger.

\end{document}